\documentclass{article}

\usepackage{microtype}
\usepackage[USenglish]{babel}
\usepackage{graphicx}
\usepackage{subfigure}
\usepackage{booktabs} 
\usepackage{amsmath}
\usepackage{amssymb}
\usepackage{amsfonts}
\usepackage{bm} 
\usepackage{dsfont} 

\usepackage{hyperref}

\usepackage{xcolor} 

\usepackage[accepted]{icml2020}

\usepackage{xspace}

\newcommand{\ie}{i.e.,\ }
\newcommand{\eg}{e.g.,\ }

\newcommand{\iid}{i.i.d.\xspace}
\newcommand{\cf}{cf.\xspace}

\usepackage{amsthm}

\newtheorem{maththm}{Theorem}
\newtheorem{mathlem}{Lemma}
\newtheorem{mathcor}{Corollary}

\newcommand{\eu}{\mathrm{e}}
\newcommand{\iu}{\mathrm{i}}
\newcommand{\upd}{\mathrm{d}}

\newif\ifdraft
\usepackage{color}

\draftfalse 
\ifdraft
\usepackage{cancel}
\fi

\newcommand{\weights}{w}
\newcommand{\dropoutVar}{z}

\newcommand{\nonlinear}{\phi}
\newcommand{\hiddenDim}{h}
\newcommand{\IR}{\mathbb{R}}
\newcommand{\prob}{\text{P}}
\newcommand{\expectation}[1]{\textbf{E}_Z\left\{#1\right\}}
\newcommand{\expectationWithoutVar}[1]{\textbf{E}\left\{#1\right\}}
\newcommand{\expectationW}[1]{\textbf{E}_W \left\{#1\right\}}

\newcommand{\eq}[1]{Eq.~(\ref{#1})}

\icmltitlerunning{Monte Carlo Dropout in Wide Neural Networks}

\begin{document}

\twocolumn[
\icmltitle{Characteristics of Monte Carlo Dropout in Wide Neural Networks}



\icmlsetsymbol{equal}{*}

\begin{icmlauthorlist}
\icmlauthor{Joachim Sicking}{equal,iais}
\icmlauthor{Maram Akila}{equal,iais}
\icmlauthor{Tim Wirtz}{equal,iais,fzml}
\icmlauthor{Sebastian Houben}{iais}
\icmlauthor{Asja Fischer}{bochum}
\end{icmlauthorlist}

\icmlaffiliation{iais}{Fraunhofer Institute for Intelligent Analysis and Information Systems IAIS, Sankt Augustin, Germany}
\icmlaffiliation{bochum}{Faculty of Mathematics, Ruhr-University Bochum, Bochum, Germany}
\icmlaffiliation{fzml}{Fraunhofer Center for Machine Learning}
\icmlcorrespondingauthor{J. Sicking}{joachim.sicking@iais.fraunhofer.de}

\icmlkeywords{Bayesian inference, Monte Carlo dropout, wide neural networks, central limit theorem}

\vskip 0.3in
]



\printAffiliationsAndNotice{\icmlEqualContribution} 

\begin{abstract}

Monte Carlo (MC) dropout is one of the state-of-the-art approaches for uncertainty estimation in neural networks (NNs). It has been interpreted as approximately performing Bayesian inference. Based on previous work on the approximation of Gaussian processes by wide and deep neural networks with random weights, we study the limiting distribution of wide untrained NNs under dropout more rigorously and prove that they as well converge to Gaussian processes for fixed sets of weights and biases. We sketch an argument that this property might also hold for infinitely wide feed-forward networks that are trained with (full-batch) gradient descent. The theory is contrasted by an empirical analysis in which we find correlations and non-Gaussian behaviour for the pre-activations of finite width NNs. We therefore investigate how (strongly) correlated pre-activations can induce non-Gaussian behavior in NNs with strongly correlated weights.

\end{abstract}

\section{Introduction}\label{sec:introduction}

Despite the huge success of deep neural networks (NNs),  robustly quantifying their prediction uncertainty is still an open problem. One of the most popular methods for uncertainty prediction is Monte Carlo (MC) dropout \cite{gal2016dropout}.
MC dropout has been proven practically successful in many applications, such as different regression tasks \cite{kendall2017uncertainties}, natural language processing~\cite{press2016using}, and object detection~\cite{miller2018dropout}. 
However, 
providing only a rough approximation of Bayesian inference,
MC dropout also carries theoretical and practical drawbacks. 
For example, as \citet{gal2017concrete} 
point out, the dropout rate has to be carefully tuned 
in order to correctly estimate the uncertainty at hand 
and
\citet{osband2016risk} showed that it does not converge to concentrated distributions in the infinite data limit. 

We therefore revisit the characteristics of the output distribution of NNs under dropout in this paper. To do so, we 
(i) present a rigorous theoretical analysis building on 
previous work for
wide feed-forward NNs. This work has established that an \iid~prior over their weights leads to approximating Gaussian processes, 
\ie drawing one set of weights yields a NN prediction 
that approximates the realization of one drawn function from a corresponding Gaussian process  initially been shown for NNs with only one hidden layer~\cite{Neal1994} 
and was recently generalized to deep NNs~\cite{lee2018deep, Matthews2018}.
%
We adapt the proof strategy of the latter work to show that performing dropout during inference in wide NNs with fixed weights also leads to an approximation of Gaussian processes (Sec.~\ref{sec:random_nn}).
(ii) Since the applicability of the central limit theorem to the pre-activations (i.e.~the inputs) of the neurons plays a central role in our proof, we
 investigate the pre-activations 
in randomly initialized and trained NNs under dropout and find
that the former are approximately normal distributed while the latter may follow non-Gaussian distributions (Sec.~\ref{sec:empirical}). (iii) We then demonstrate that such non-Gaussian distributions can be
induced by correlations of the pre-activations in the previous layer, by showing that these  lead to
distributions with exponential tails in a toy model and NNs with strongly correlated weights (Sec.~\ref{sec:strongly_correlated}).

\section{Theoretical Analysis of Random Neural Networks}\label{sec:random_nn}

As a first step towards a characterization of the limiting process of NNs under MC dropout at inference time we rigorously study a random NN. We consider a feed-forward NN  with $k+2$ layers which, for an input vector $x\in \IR^d$, is parameterized as follows: 
\begin{align}
 f^{(\nu)}_i(x) &= \sum_{j=1}^{h_{\nu-1}(n)} \frac{\weights_{ij}^{(\nu)} g^{(\nu-1)}_j(x)}{\sqrt{h_{\nu-1}(n)}} + b_i^{(\nu)}~,  \label{eq:def_f_mu_j_paper}
\end{align}
where $g^{(\nu)}_i(x)=\dropoutVar^{(\nu)}_{i}\nonlinear(f^{(\nu)}_i(x))$ for $\nu=1,\dots,k$, $g^{(0)}_i(x)\allowbreak = \allowbreak x_i$. $\phi$ is the activation function, $z_i$ are the dropout Bernoul\-li random variables (RVs) with keep rate $q=1-p$, and $h_{\nu}(n)$ is the neuron count in layer $\nu$.\footnote{A more detailed overview of the  structure and notation of this feed-forward NN is summarized in Appendix~\ref{app:clt_rnn}.} 

In previous work, see 
\citet{Matthews2018,lee2018deep,wu2018deterministic,tsuchida2019richer},
the authors considered NNs with independent prior distributions on the NN parameters. The resulting summands in Eq.~\eqref{eq:def_f_mu_j_paper} are independent RVs such that a central limit theorem can readily be applied. In contrast we examine a NN with a \emph{fixed set} of parameters and dropout acting as independent prior distribution on the activations. As a consequence, the resulting summands in Eq.~\eqref{eq:def_f_mu_j_paper}  are independent RVs only in the second layer ($\nu=2$) 
and are dependent RVs for all later ones. This follows from an inspection of the covariance of the pre-activations yielding
\begin{align}
\begin{split}
&\text{Cov}\left(f^{(\nu)}_i,f^{(\nu)}_j\right)= 
\frac{q}{h_{\nu-1}}\sum_{k,l=1}^{h_{\nu-1}(n)} \weights_{il}^{(\nu)}\weights_{jk}^{(\nu)}\\
&\times \Big( \left( \delta_{kl} + q\ (1-\delta_{kl})\right) \expectation{ \nonlinear^{(\nu-1)}_k \nonlinear^{(\nu-1)}_l} \\
& - q\,\expectation{ \nonlinear^{(\nu-1)}_k}\expectation{\nonlinear^{(\nu-1)}_l} \Big) \enspace,
\end{split}\label{def_variance_fi_fj}
\end{align}
where we dropped the arguments of $f^{(\nu)}_i$ and $\nonlinear^{(\nu-1)}_l$ for simplicity of notation. For finite $h_{\nu-1}(n)$ and $i\neq j$ the expression above is in general non-zero. Thus, 
we cannot apply standard central limit theorems.
Nonetheless, we are able to show that a random NN under MC dropout converges to a Gaussian process in the limit of infinite width. Our analysis facilitates the property of weights and biases being realizations of independent Gaussian distributed RVs in order to bound their behavior if the width of the NN approaches infinity. We summarize our findings and main result of this section in the following theorem.
\begin{maththm}\label{infinite_nn_with_dropout_are_Gaussian_processes}
Let $f^{(k)}(x)\in\IR^{L}$ be the pre-activation vector of the $k+2$-th layer of a NN, as defined by Eq.~\eqref{eq:def_f_mu_j_paper}
, where weights and biases follow a Gaussian distribution with zero mean and unit variance. For the widths of the NN layers going simultaneously to infinity, for each bounded $x\in\IR^d$ with $||x||_2^2\leq \alpha\,d$, and for a fixed set of weight and bias variables, the pre-activations of all layers $\nu=2,\dots,k+1$ converge in distribution over MC dropout to an uncorrelated Gaussian process.
\end{maththm}
The theorem is a direct consequence of corollary~\ref{corollary:activations_are_Gaussian_Processes} and is 
proven in Appendix~\ref{app:clt_rnn}. The following observation regarding  the mutual correlations between $f_i^{(\nu)}(x)$ and $f_j^{(\nu)}(x)$ is central to the proof: since weights and biases are \iid Gaussian RVs, the correlations between the pre-activations are rather weak and vanish for $n\to\infty$ whenever $i\neq j$.
Extending  theorem~\ref{infinite_nn_with_dropout_are_Gaussian_processes} to trained NNs can be achieved using that for (full-batch) gradient descent and rather weak assumptions on the loss function, the by $1/\sqrt{n}$ rescaled distance ($n$ being a control parameter of the width, see Appendix~\ref{app:clt_rnn}) of weights at initialization time and at (a finite) optimization step $t$, $||W(0)-W(t)||_{op}/\sqrt{n}$, converges uniformly to zero for $n\to\infty$ on $t\in[0,T]$~\cite{du2018gradient,jacot2018neural}.
This indicates that we can expect to find Gaussian processes  also for infinitely wide, \emph{trained}  NNs with MC dropout (we leave a rigorous proof of this sketched argument for future work). This theoretical insight stands in contrast to the non-Gaussian behavior of trained wide NNs we obtain in the experiments outlined in the subsequent section. We therefore will focus the investigation of the implications  (strongly) correlated (pre-)activations in Sec.~\ref{sec:strongly_correlated}.


\section{Ambiguous Empirical Observations}\label{sec:empirical}

To substantiate our discussions in Secs.~\ref{sec:random_nn} and \ref{sec:strongly_correlated} empirically, we study the pre-activations of a fully connected NN of the form given in Eq.~\ref{eq:def_f_mu_j_paper}.
We train it for classification on FashionMNIST \cite{xiao2017fashion} using a cross-entropy loss and
hyperbolic tangent (tanh) non-linearities. We employ a \textit{narrow} network $\mathcal{H}_{\rm narrow}$ with $k = 9$ hidden layers of width $h_\nu(n) = h = 100$ each as well as a \textit{wide} network $\mathcal{H}_{\rm wide}$ with $k = 7$ and $h = 1000$. All network weights are initialized with \iid~random values. We train for $100$ epochs using the Adam \cite{kingma2014adam} optimizer and an initial learning rate of $0.001$ (see Appendix \ref{app:empirical} for more implementation details). Bernoulli dropout with $p = 0.2$ is applied to the activations of all hidden network layers.

\begin{figure}[bth]
    \centering
    \includegraphics[width=0.4\textwidth]{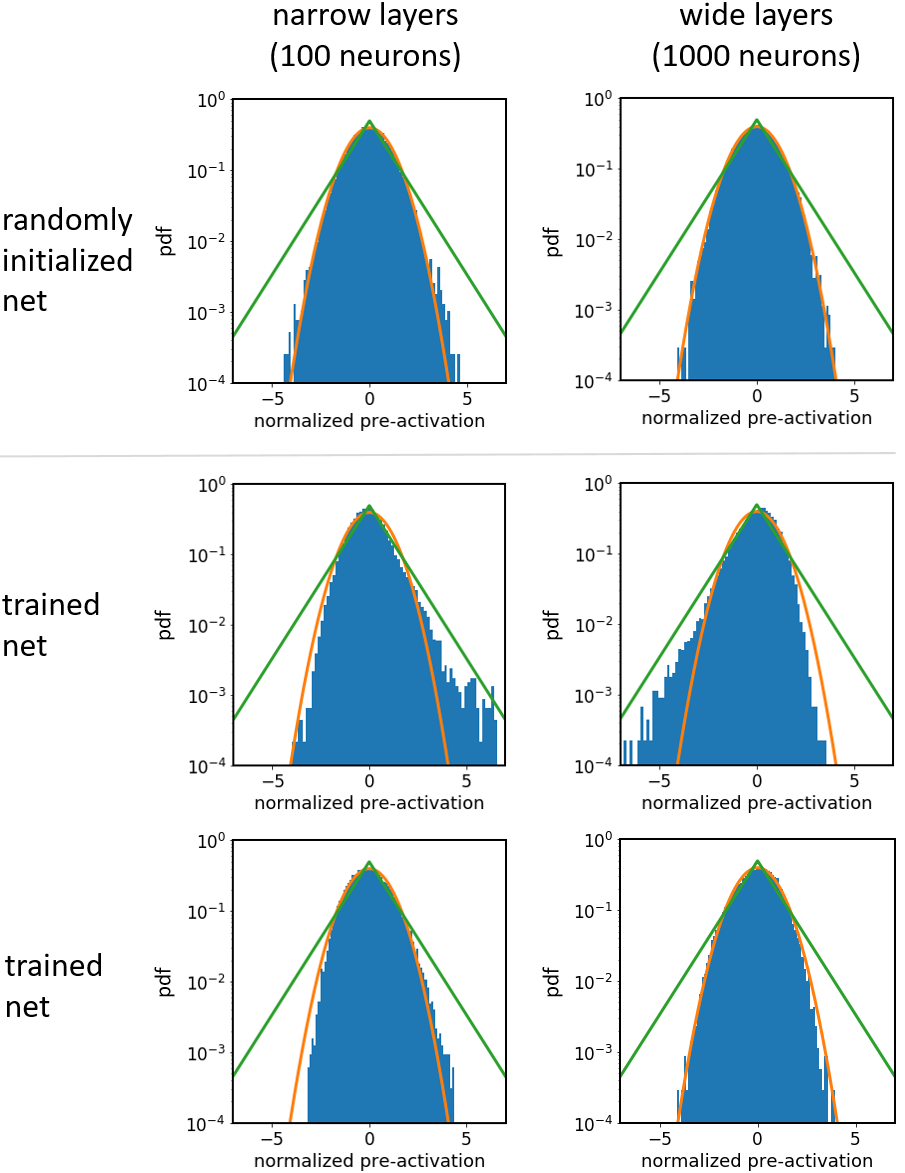}
    \caption{Normalized pre-activation distributions of selected neurons from different networks. We consider randomly initialized (top row) and trained (middle and bottom row) nets that are either narrow ($h = 100$, left column) or wide ($h = 1000$, right column). We handpicked the neurons from the trained networks (see text). Standard normal (orange) and double exponential (green) densities are given for comparison.}
    \label{fig:pre_act_distributions}
\end{figure}
First, we study the pre-activation distributions of the untrained $\mathcal{H}_{\rm narrow}$ and $\mathcal{H}_{\rm wide}$. These distributions are generated by running $30{,}000$ forward passes with dropout for a fixed test image. We randomly pick one neuron from the last hidden layer of each network and visualize its pre-activation distribution in Fig.~\ref{fig:pre_act_distributions} (top row). To ease visual comparison between different neurons, we report normalized pre-activations. The resulting distributions are clearly Gaussian, compare our discussion in Sec.~\ref{sec:random_nn}. The weight correlations and pre-activation correlations of these untrained networks are studied in Appendix \ref{app:empirical}.

Analyzing the trained networks $\mathcal{H}_{\rm narrow}$ and $\mathcal{H}_{\rm wide}$ yields more complex results. The pre-activation distributions for exemplarily selected hidden units from the last hidden layer of the trained $\mathcal{H}_{\rm narrow}$ and $\mathcal{H}_{\rm wide}$ are shown in Fig.~\ref{fig:pre_act_distributions} (middle and bottom row). The 
neurons are handpicked to illustrate the variety of distributions we encounter in this layer. We observe Gaussian as well as non-Gaussian pre-activation distributions that range over skewed Gaussians to clearly non-Gaussian, exponential ones
\footnote{Technically also the tails of a Gaussian decay exponentially, however with $\exp(-\xi^2)$. Throughout this text we refer to exponential tails only for the cases of $\exp(-|\xi|)$ or slower decay.}.
For $\mathcal{H}_{\rm wide}$, approximately $40\%$ of all neurons have Gaussian distributions, $40\%$ skewed Gaussians and $20\%$ exponential ones. Each histogram in \mbox{Fig.~\ref{fig:pre_act_distributions}} is based on $30{,}000$ forward passes with dropout.

Moreover, we find these observations to depend on the input image: training inputs yield less tailed distributions than test inputs that in turn lead to less pronounced tails than artificial ones that are superpositions of two training images. 
However, there are large differences within those input categories, \eg between test images. Shuffling the trained weight matrices yields results that are similar to the randomly initialized case. This indicates that weight and pre-activation dependencies are important while changes of the marginal distributions due to model training are not. While these results apply for the \textit{last} hidden layer, we observe mostly Gaussian or slightly skewed Gaussians in earlier layers. In Sec.~\ref{sec:strongly_correlated}, we demonstrate on a toy model that non-Gaussian properties accumulate during the propagation through the network. Further empirical observations and an analysis of the weight correlations and pre-activation correlations of the trained networks can be found in Appendix \ref{app:empirical}.

These complex dependence structures are the result of an interplay of weight distributions, input-immanent structures, and network non-linearities that are orchestrated by network training. A theoretical foundation for the Gaussian results was laid out in the previous section. For the deviations we take a closer look at correlations in the next section.


\section{Completing the Picture - Strongly Correlated Systems}
\label{sec:strongly_correlated}

As we have seen in \eq{def_variance_fi_fj} the (pre-)activations from any layer $\nu$ are, in general, not independent from one another.
We explore possible consequences of this observation in terms of consecutive toy experiments.
For this, the product of \textit{two} RVs, $Z = XY$, is central.
On an abstract level, $Y$ represents a random activation from a previous layer and $X$ the product $w_{ij} z_j$.
For later simplicity, we model both terms as Gaussian, $X,Y\sim \mathcal{N}(\mu,\sigma)$.
In the case of vanishing mean, $\mu=0$, one obtains the well known analytic form for the PDF of $Z$,
\begin{equation}
    \operatorname{PDF}_{Z}(\xi) = \frac{1}{\pi} K_0(|\xi|) \qquad (\sigma=1) \enspace,
    \label{eq:zN1PDF}
\end{equation}
with the modified Bessel function $K_0$.
Derivations and fur\-ther comments are relegated to Appendix \ref{app:toy_model}.
In contrast to the normal distributions its asymptotic,
\begin{equation}
    \operatorname{PDF}_{Z}(\xi) \sim \frac{1}{\sqrt{2\pi |\xi|}} \eu^{-|\xi|} \qquad (\sigma=1,\,|\xi|\gg 1) \enspace,
    \label{eq:zN1PDFasympt}
\end{equation}
reveals exponentially decaying tails.

For any given neuron we encounter sums over $h$ terms, where $h$ denotes the layer width, instead of single products.
If these summands $X_i Y_i$ were independent as argued for in the first section, one would recover a normal distributed outcome for the neuron by the central limit theorem.
But if they are correlated, the result may  differ drastically.
The outcome strongly depends on the respective covariance matrices of $X_i,Y_i$, for details see Appendix \ref{app:toy_model}.
As a rule of thumb, larger correlations favor non-Gaussian results.
This can be illustrated in terms of a toy extension choosing
\begin{equation}
    Z=\sum_{i=1}^h X_i Y_i\,,
    \quad
    X_i = c\,x_0 + (1-c)\,x_i
    \label{eq:corrToy}
\end{equation}
and similarly for $Y_i$ with Gaussian RVs $x_\gamma,y_\gamma$ for $\gamma = 0,\allowbreak 1,\ldots h$.
The factor $c$ controls a global correlation among the entries.
For this case we find the decomposition
\begin{equation}
    \begin{split}
        \sum_{i=1}^h X_i Y_i = &(1-c)^2 \sum_{i=1}^h x_i\,y_i+c^2 h x_0\,y_0
        \\
        & + c\,(1-c)\left(x_0 \sum_{i=1}^h y_i + y_0 \sum_{i=1}^h x_i \right)\enspace.
    \end{split}
\end{equation}
Heuristically speaking, the first term is of Gaussian nature while the second (and later ones) contribute exponential tails, see the limits for $c=0$ or $1$, respectively.
An important observation is that both terms are on similar footing with respect to $h$, \ie also in the large $h$ limit the Gaussian term will not suppress the first one.
This finding is due to the choice of a global correlation present in all $h$ terms of $x_i$ and $y_i$.

Assuming that the input $Y$ has exponential tails, we can investigate how those propagate for the $Z=XY$ model (given $\mu_x=0$).
For this, we take
\begin{equation}
    \operatorname{PDF}_{Y}(\xi) \propto \frac{1}{|\xi|^{(n-1)/n}}\exp{\left(-\alpha |\xi|^{2/n}\right)}
    \qquad n\in\mathbb{N}
    \enspace 
\end{equation}
as an analytically tractable approximation capturing the tail behavior.
For $n=1$, \ie Gaussian-like decay, we obtain a Bessel result for $Z$ comparable to Eq.~\eqref{eq:zN1PDF}, in which the tail is modeled by $n=2$, see Eq.~\eqref{eq:zN1PDFasympt}.
More generally, for small values of $n$ we obtain the explicit asymptotics for $Z$ as:
\begin{equation}
    \operatorname{PDF}_{Z}(\xi) \sim \frac{1}{|\xi|^{n/(n+1)}}\exp{\left(-\kappa^{1/(n+1)} |\xi|^{2/(n+1)} \right)}
    \label{eq:toyGenAsympt}
\end{equation}
with $|\xi|\gg 1 $ and some positive constant $\kappa > 0$ depending on $\alpha^n/\sigma_X^2$.
As the decay slows down with each iteration, this effect might contribute to an ``accumulation'' of tails.

Inspired by the presented heuristics, we revisit the randomly initialized NN introduced in Sec.~\ref{sec:empirical}.
However, we initialize the weight matrices in a \textit{correlated} fashion similar to the toy model in \eq{eq:corrToy}, see Appendix \ref{app:toy_model}.
The resulting distributions for the pre-activations given random input are shown in Fig.~\ref{fig:net_with_correlated_init}.
With $c=0.1$ the correlation is deliberately small and the dropout rate $p=0.2$ set as before, but we investigate deeper networks of $L=50$ layers.
The left panel shows that for $h=1000$ only the initial layer follows a Gauss\-ian behavior and all later ones exhibit exponential tails of increasing strength. 
A closer look reveals that the decay, especially for later layers, is slightly weaker than exponential, which might be caused by the effects presented in Eq.~\eqref{eq:toyGenAsympt}.
On the right hand side we compare the pre-activations of the 13\textsuperscript{th} layer for different network widths, from shallow ($h=100$) to wide ($h=2000$).
Except for small fluctuations of the tails, which besides statistics are caused by choosing different (fixed) random inputs for each network, the result is stable with $h$, suggesting that also the infinite width limit will not lead to a Gaussian outcome.
\begin{figure}[bth]
    \centering
    \includegraphics[width=1\columnwidth]{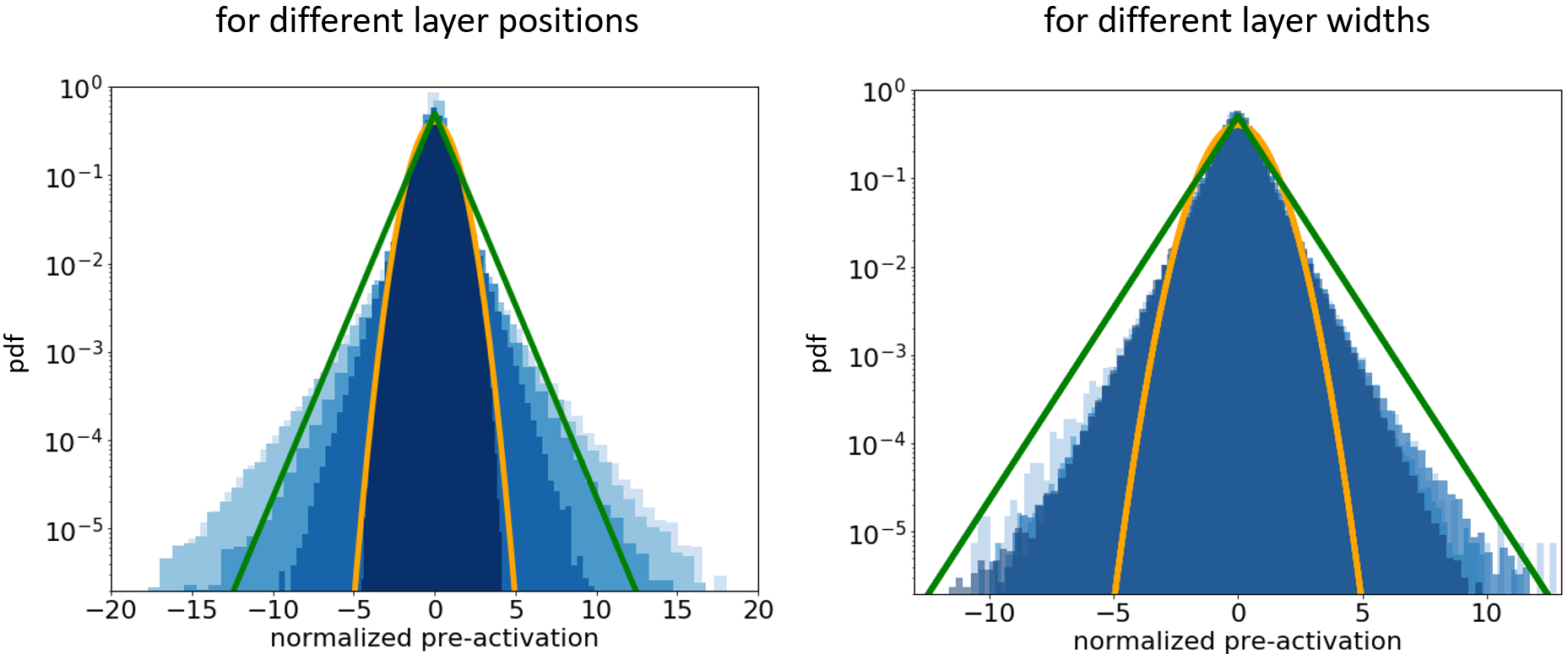}
    \caption{Normalized pre-activation distributions from networks with correlated random weights, details in text. On the l.h.s.~we color-code hidden layer position, first (dark blue) to last (light blue). The r.h.s.~shows different network widths (narrow (light blue) to wide (dark blue)) for layer 13. Gaussian (orange) and exponential tailed distribution (green) are shown for comparison.}
    \label{fig:net_with_correlated_init}
\end{figure}

Figs.~\ref{fig:weight_correlations} and \ref{fig:pre_act_correlations} in Appendix \ref{app:empirical} show the empirical correlations of the trained network from Sec.~\ref{sec:empirical}.
While their structure is more involved, we find that the correlation of the pre-activations increases for deeper layers.
More importantly, both correlations do not decrease to zero with increased layer width, which satisfies a central assumption of the toy model here.
It might therefore be less surprising that the trained network can exhibit exponentially tailed pre-activations, \cf Fig.~\ref{fig:pre_act_distributions}.

\section{Discussion}\label{sec:discussion}


We rigorously proved that random neural networks with fixed parameters under dropout converge to Gaussian processes in the limit of infinitely wide layers. A sketched argumentation fuels hopes that a similar behavior can be shown for weakly correlated, infinitely wide networks that are trained with (full-batch) gradient descent. 

Empirically studying wide (trained) networks with a rich dependence structure reveals a more complex picture: the coexistence of Gaussian and non-Gaussian, exponential, latent distributions. We shed light on this observation using a simple toy model and a network with correlated random initialization. These two systems indicate the existence of (at least) two regimes: one of \emph{weakly} correlated pre-activations with Gaussian distributions and another one of \emph{strongly} correlated pre-activations with exponentially tailed limiting distribution functions. Future work will investigate these two regimes in more detail to understand how to deliberately manipulate the properties of these limiting distributions under MC dropout. Searching for the borders of these regimes and further classes of limiting distributions is another research path.

Our empirical observation of Gaussianity that decreases from training data over test data to out-of-distribution data, gives rise to both theoretical and applied future work: firstly, to better understand the generalization properties of the \textit{learned} Gaussian or non-Gaussian behavior and, secondly, to employ this knowledge to construct more robust uncertainty quantifications.

\textit{Acknowledgement} --- The research of the authors M.\ Akila, S.\ Houben and J.\ Sicking was funded by the German Federal Ministry for Economic Affairs and Energy within the project ``KI Absicherung – Safe AI for Automated Driving''. Said authors would like to thank the consortium for the successful cooperation. The work (T.\ Wirtz) was developed by the Fraunhofer Center for Machine Learning within the Fraunhofer Cluster for Cognitive Internet Technologies. This work of A.~Fischer was supported by the Deutsche Forschungsgemeinschaft (DFG, German Research Foundation) under Germany’s Excellence Strategy – EXC-2092 \textsc{CaSa} – 390781972.




\bibliography{references}
\bibliographystyle{icml2020}


\newpage

\appendix
\section{Central Limit Theorem for Random Neural Networks}\label{app:clt_rnn}

The network architectures considered here, are  fully connected and therefore defined by the width of the hidden layers, which we will denote by $\hiddenDim_0,\dots,\hiddenDim_{k+1}$. 
The depth and the size of the input and the output layer of the NN, 
$\hiddenDim_0 = d$ and $\hiddenDim_{k+1}=L$, respectively, are kept fix during the analysis. More formally, we define a set of \emph{width functions} $h_{\nu}:\mathbb{N}\rightarrow \mathbb{N}, ~n\mapsto h_{\nu}(n), \text{for }\nu=1,\dots,k$ in which $n$ is a control parameter meaningless for actual applications and $h_{\nu}(n)$ specifies the number of hidden neurons at layer $\nu$ for a given input $n$. We consider the case where $n\rightarrow\infty$ implies $h_\nu(n)\to\infty$ and $h_{\nu}(n)/n < \infty$ for all $n$. Without loss of generality we assume the same non-linear activation function in each layer, denoted by $\nonlinear:\IR\rightarrow\IR$, $\phi$ to be (piecewise) Lipschitz continuous, i.e.~$|\phi(u)-\phi(v)| \leq K|u-v|$ for a fixed $K\in\IR_+$ and all $u,v\in\IR$, and that the weights $w_{ij}^{(\nu)}$ and biases $b_{ij}^{(\nu)}$, 
are drawn from a standard normal 
distribution $\mathcal{N}(0,1)$. 
Let $x\in \IR^d$ be a bounded input vector to the NN, that we parametrize as follows
\begin{align}
 f^{(\nu)}_i(x) &=\frac{1}{\sqrt{h_{\nu-1}(n)}} \sum_{j=1}^{h_{\nu-1}(n)} \weights_{ij}^{(\nu)} g^{(\nu-1)}_j(x) + b_i^{(\nu)}~,  \label{eq:def_f_mu_j}\\
 g^{(\nu)}_i(x) &=\left\{
                \begin{array}{ll}
                 \dropoutVar^{(\nu)}_{i}\nonlinear\left(f^{(\nu)}_i(x)\right)&,\nu=1,\dots,k-1\\
                 x_i&,\nu=0
                \end{array}
              \right.~,\label{eq:def_g_mu_j}
\end{align}\label{eq:def_f_0_j}
where $i=1,\dots,h_{\nu}(n)$, $\nu=0,\dots,k$, and $\dropoutVar^{(\nu)}_{i} \in \{0,1\}$ are independent binary random variables with $\prob(\dropoutVar^{(\nu)}_{j} = 1) = q$ for all $i,j,\nu$. For analytic tractability, we do not consider dropout in the input layer because its dimension is kept fixed for $n\to \infty$.

The proof that random neural networks with MC dropout converge to a Gaussian Processes relies on the same mathematical model as proposed in \cite{Matthews2018}, which we summarize for completeness. In order to study the limiting distribution of the pre-activations, in layer $\nu$, we set its width to infinity, $h_\nu(n)=\infty$ and embed it into $\IR^\infty$. Convergence in distribution in this abstract space can be studied only with respect to a certain topology. Following~\cite{Billingsley:prob_and_measure} this topology is generated by a metric $\rho$:
\begin{align*}
    \rho(u,v) = \sum_{i=1}^\infty \frac{\text{min}\left\{1,|v_i-u_i|\right\}}{2^i}
\end{align*}
According to~\cite{Billingsley:prob_and_measure} this metric metricise the product topology of the product of countable many copies of $\IR$ with the usual Euclidean topology~\cite{Dashti2017}. Moreover, to prove weak convergence in such an abstract space, it is sufficient to prove weak convergence for each finite dimensional marginal. To study the convergence behavior of the finite dimensional marginals simultaneously~\cite{Matthews2018} suggested to use the Cramér-Wold device.  
\begin{maththm}
For random vectors $X_n=(X_{n1},\dots,X_{nM})$ and $Y=(Y_{1},\dots,Y_{M})$, a necessary and sufficient condition for a $X_n$ converging in distribution to $Y$ ($X_n \Rightarrow Y$) is that  
\begin{align}
    \sum_u t_u X_{nu} \Rightarrow \sum_u t_u Y_{u}~, \quad \forall~t \in \IR^M~.
    \label{eq:projective_convergence}
\end{align}
That is, if every linear combination of the coordinates of $X_n$ converges in distribution to the correspondent linear combination of coordinates of $Y$.
\end{maththm}
\noindent This theorem reduces the problem of weak convergence of the distribution of the finite marginals, to the convergence of the distribution of a one-dimensional random variable. Based on the Cramér-Wold theorem, we will show that 
\begin{align}
\psi_{\nu}(t) &= \frac{1}{s_n}\sum_{u=0}^M t_u\left( f^{(\nu)}_u(x) - \expectation{f^{(\nu)}_u(x)}\right)\\ 
\begin{split}
&=\frac{1}{s_n}\sum_{j=1}^{h_{\nu-1}(n)} \frac{1}{\sqrt{h_{\nu-1}(n)}}\sum_{u=0}^M t_u  \weights^{(\nu)}_{uj}\\&\times\left(\dropoutVar^{(\nu)}_{j}g^{(\nu-1)}_j(x)-q\expectation{g^{(\nu-1)}_j(x)}\right)
\end{split}
\label{eq:def_psi}\\
&\equiv\frac{1}{s_n}\sum_{j=1}^{h_{\nu-1}(n)}\gamma^{(\nu)}_{j,n}\,, \label{eq:def_psi_gamma}
\end{align}
where $\expectation{\psi_{\nu}(t)}=0$ and
\begin{align}
&s_n^{(\nu)2}=\sum_{j=1}^{h_{\nu-1}(n)}\sigma_{jn}^{(\nu)2}\label{def_psi_diag_variance}\\
    \begin{split}
        &=\sum_{j=1}^{h_{\nu-1}(n)} \sum_{u=0,v=0}^M \frac{t_u t_v\weights^{(\nu)}_{uj}\weights^{(\nu)}_{vj}}{h_{\nu-1}(n)}\\&\times\left(q\expectation{\left(g^{(\nu-1)}_j(x)\right)^2}-q^2\mathbf{E}_Z^2\left\{g^{(\nu-1)}_j(x)\right\}\right)~.
    \end{split}
        \label{def_psi_diag_variance_exp}
\end{align}
In the expression above we set $\mathbf{E}\{{\gamma^{(\nu)}_{j,n}}^2\}\equiv\sigma_{jn}^{(\nu)2}$. As defined above, $\psi_{\nu}(t)$ is a random variable with mean zero and unit variance. As mentioned earlier, see Eq.~\eqref{def_variance_fi_fj}, the pre-activations have a non-trivial correlation structure. Thus the activations are dependent random variables. Since the sums in the pre-activations are over dependent random variables, standard central limit theorems can not be applied. However, due to the \emph{self-averaging} property of the weights of the random neural network, we can study the behavior of the variables when the width goes to infinity. We summarize our observations in the following lemma. 
\begin{mathlem}\label{lemma:asymptotic_independence_gamma}
Let the pre-activations $\{ f^{(\nu-1)}_u(x)\}_{u\in\mathbb{N}}$ converge, for $n\to\infty$ in distribution, to independently distributed random variables, let $ f^{(\nu)}_i(x)$, $ g^{(\nu)}_i(x)$, $s_n^{(\nu)2}$ and $\gamma^{(\nu)}_{j,n}$ be as defined in Eq.~\eqref{eq:def_f_mu_j}, \eqref{eq:def_g_mu_j}, \eqref{def_psi_diag_variance} and \eqref{eq:def_psi_gamma}, respectively. Then it follows that $\psi^{(\nu)}(t)$ converges in distribution to the distribution of  $\psi^{(\nu)}(t)$ with $\gamma^{(\nu)}_{j,n}$ replaced by a random variable drawn from the \emph{marginalized distribution} of $\gamma^{(\nu)}_{j,n}$.

\end{mathlem}
\begin{proof}
According to the discussion above we prove that $\psi^{(\nu)}(t)$, as, defined in Eq.~\eqref{eq:def_psi}, converges to a Gaussian distributed random variable with mean zero and variance one. We first show that $\psi^{(\nu)}(t)$ converges in distribution to the distribution of  $\psi^{(\nu)}(t)$ with $\gamma^{(\nu)}_{j,n}$ replaced by a random variable drawn from the \emph{marginalized distribution} of $\gamma^{(\nu)}_{j,n}$. To prove this, we have to show that 
 \begin{align}
     &\left|\expectation{\prod_{j=1}^{h_{\nu-1}(n)}e^{\iu \lambda \gamma^{(\nu)}_{j,n}/s_n }}- \prod_{j=1}^{h_{\nu-1}(n)}\expectation{e^{\iu \lambda \gamma^{(\nu)}_{j,n}/s_n }} \right| \label{absolute_value_psi_with_without_margin}
 \end{align}
 converges to zero for $n\to\infty$. To simplify the expression above, we combine the following estimate~\cite{Billingsley:prob_and_measure} of the exponential function
 \begin{align}
 |(1+\iu \lambda x - \frac{\lambda^2}{2} x^2 ) - e^{\iu \lambda x}| \leq \text{min}\{|x\lambda|^2,|x\lambda|^3\} \label{upper_bound_complex_exp}
 \end{align}
with the following observation  
\begin{align}
\begin{split}
&\left|\expectation{\prod_{j=1}^{h_{\nu-1}(n)}e^{\iu \lambda \gamma^{(\nu)}_{j,n}/s_n }} -c \right|\leq \\& \expectation{\left|\left(1+\frac{\iu \lambda}{s_n} \gamma^{(\nu)}_{i,n} - \frac{\lambda^2}{s_n^2 2} {\gamma^{(\nu)}_{i,n}}^2 \right)- e^{\iu \lambda \gamma^{(\nu)}_{i,n}/s_n }\right|} + \\&\left|\expectation{\left(1+\frac{\iu \lambda}{s_n} \gamma^{(\nu)}_{i,n} - \frac{\lambda^2}{s_n^2 2} {\gamma^{(\nu)}_{i,n}}^2 \right)\prod_{j\neq i}^{h_{\nu-1}(n)}e^{\iu \lambda \gamma^{(\nu)}_{j,n}/s_n }} - c \right| ~,\label{upper_bound_characteristic_func_psi}
\end{split}
\end{align}
where $c\in\mathbb{C}$. It holds for all $i\leq h_{\nu-1}(n)$ and is a simple consequence of the  fact that the characteristic function has absolute value one. Applying Eq.~\eqref{upper_bound_characteristic_func_psi} and~\eqref{upper_bound_complex_exp} to both terms  in expression~\eqref{absolute_value_psi_with_without_margin} multiple times, we end up with the following upper bound on Eq.~\eqref{absolute_value_psi_with_without_margin},
 \begin{align}
     &\left|\expectation{\prod_{j=1}^{h_{\nu-1}(n)}e^{\iu \lambda \gamma^{(\nu)}_{j,n}/s_n }}- \prod_{j=1}^{h_{\nu-1}(n)}\expectation{e^{\iu \lambda \gamma^{(\nu)}_{j,n}/s_n }} \right| \nonumber \\
     &\leq \sum_{j=1}^{h_{\nu-1}(n)} 2\expectation{\text{min}\left\{\left|\frac{\gamma^{(\nu)}_{j,n}\lambda}{s_n}\right|^2,\left|\frac{\gamma^{(\nu)}_{j,n}\lambda}{s_n}\right|^3\right\}}\label{expectation_Talyor_complex_exponent}\\
     \begin{split}
     +\left|\expectation{\prod_{j=1}^{h_{\nu-1}(n)}\left(1+\frac{\iu \lambda \gamma^{(\nu)}_{j,n}}{s_n} - \frac{\lambda^2}{s_n^22} {\gamma^{(\nu)}_{j,n}}^2 \right)} \right.\\\left.- \prod_{j=1}^{h_{\nu-1}(n)}\expectation{\left(1+\frac{\iu \lambda \gamma^{(\nu)}_{j,n}}{s_n}- \frac{\lambda^2}{2s_n^2} {\gamma^{(\nu)}_{j,n}}^2 \right)} \right|~.\label{second_order_approx}
     \end{split}
 \end{align}
The sum~\eqref{expectation_Talyor_complex_exponent} in the expression above, appears also in the proof of the Lindeberg central limit theorem. Following~\cite{Billingsley:prob_and_measure} let $\epsilon>0$ be positive, then due to $\text{min}$-function in the expectation, the sum~\eqref{expectation_Talyor_complex_exponent} is at most
\begin{align}
\begin{split}
  &\frac{|\lambda|^2}{s_n^2}\sum_{j=1}^{h_{\nu-1}(n)}\int_{|\gamma^{(\nu)}_{j,n}|>\epsilon} \text{dP}_j(\gamma^{(\nu)}_{j,n})~ \gamma^{(\nu)2}_{j,n}  \\+&\frac{|\lambda|^3}{s_n^3}\sum_{j=1}^{h_{\nu-1}(n)}\int_{|\gamma^{(\nu)}_{j,n}|<\epsilon} \text{dP}_j(\gamma^{(\nu)}_{j,n}) \left|\gamma^{(\nu)}_{j,n}\right|^3 
  \end{split}\\
  \leq & \epsilon \frac{|\lambda|^3}{s_n} + \frac{|\lambda|^2}{s_n^2}\sum_{j=1}^{h_{\nu-1}(n)}\int_{|\gamma^{(\nu)}_{j,n}|>\epsilon} \text{dP}_j(\gamma^{(\nu)}_{j,n})~ \gamma^{(\nu)2}_{j,n}~,\label{appearence_Lindeberg_condition}
\end{align}
where $\text{dP}_j(\gamma^{(\nu)}_{j,n})$ is the marginal probability measure of $\gamma^{(\nu)}_{j,n}$. Setting $\epsilon\to\epsilon s_n$ the second term in Eq.~\eqref{appearence_Lindeberg_condition}, is exactly the Lindeberg condition. In Lindeberg-CLT a necessary condition for the in distribution convergence to a Gaussian distributed random variable is that the Lindeberg-Condition vanishes for $n\to\infty$. A necessary condition for the Lindeberg-condition to be satisfied is that the Lyapunov-condition~\eqref{eq:Lyapounov_condition} holds~\cite{Billingsley:prob_and_measure}. The latter is proven to be satisfied  in lemma~\ref{proof_satisfaction_of_lypunov_condition}. Thus, since  second term in expression~\eqref{appearence_Lindeberg_condition} converges to zero for $n\to\infty$ and that $\epsilon$ is arbitrary, the sum~\eqref{expectation_Talyor_complex_exponent} converges to zero.

To conclude the proof, it remains to show that expression~\eqref{second_order_approx} converges to zero for $n\to\infty$. First of all, we observe that $\mathbf{E}_Z\{1+\iu \lambda \gamma^{(\nu)}_{j,n}/s_n- \lambda^2 {\gamma^{(\nu)}_{j,n}}^2/2s_n^2 \}=1-\lambda^2\sigma_{jn}^{(\nu)2}/2s_n$. Expanding the products in Eq.~\eqref{second_order_approx}, it turns out that we have to show that the absolute values of
\begin{align}
    \sum_{i\in D} \expectation{\prod_{j=1}^{l_1}\gamma^{(\nu)}_{i_j,n}\prod_{k=l_1+1}^{l_1+l_2}\gamma^{(\nu)2}_{i_k,n}}~,\label{mean_zero_symmetric_polyI}
\end{align}
where $D = \{i=(i_1,\dots,i_{l_1+l_2})| 1\leq i_{1}< \dots <i_{l_1+l_2} \leq h_{\nu-1} \}$, converges to zero for $n\to\infty$, for all $l_1, l_2$ such that $l_1>1$ and $l_1+l_2\leq h_{\nu-1}$ as well as that the absolute value of
\begin{align}
    \sum_{i\in D'}\left(\expectation{\prod_{j=1}^{l}\gamma^{(\nu)2}_{i_j,n}}-\prod_{j=1}^{l}\sigma_{jn}^{(\nu)2}\right)~,\label{mean_zero_symmetric_polyII}
\end{align}
where $D' = \{i=(i_1,\dots,i_{l})| 1\leq i_{1}< \dots <i_{l} \leq h_{\nu-1} \}$, converges to zero for $n\to\infty$, for all $l$ such that $l\leq h_{\nu-1}$. Both claims follow similar argumentation. Due to the fact that the weights are Gaussian distributed random variables with zero mean and unit variance, it immediately follows that the mean with respect to the weights of layer $\nu$ is zero. Thus, we can apply Chebyshev's inequality $\mathbf{P}_W(|f(W)| \leq \sqrt{\text{Var}(f(W))/\delta})\geq 1-\delta $ to obtain a high probability bound on both expressions. In both cases the variance becomes a double sum over either $D$ or $D'$ which is at most of order $\mathcal{O}(h_{\nu-1}^{2l_1+2l_2})$, the summands contributed a factor of order $\mathcal{O}(h_{\nu-1}^{-l_1-2l_2})$ times the following average over the weights of layer $\nu$
\begin{align*}
    &\textbf{E}_W\left\{\prod_{o=1}^{l_1}w^{(\nu)}_{u_oi_o}w^{(\nu)}_{u_o'j_o}\prod_{b=l_1+1}^{l_1+l_2} w^{(\nu)}_{u_bi_b}w^{(\nu)}_{u_b'i_b} w^{(\nu)}_{u_o''j_b}w^{(\nu)}_{u_o'''j_b} \right\}~,
\end{align*}
where $ 1\leq i_{1}< \dots <i_{l_1+l_2} \leq h_{\nu-1} $ and $ 1\leq j_{1}< \dots <j_{l_1+l_2} \leq h_{\nu-1}$. Due to the Gaussian nature of the weights, the leading order contributions of the expression above are due to the second are of the form $\prod_{o=1}^{l_1}\delta_{i_oj_o}$. The Kronecker-Deltas lead to a reduction of the order of the sum, namely the reduce contributions of the double sum by a factors of $\mathcal{O}(h_{\nu-1}^{l_1})$. Thus, there exist a constant $C$ such that the variance of  expression~\eqref{mean_zero_symmetric_polyI} can be bounded, with at least probability $1-\delta$, by
\begin{align*}
&\frac{C}{\sqrt{\delta}}\left(\frac{||t||_2}{s_n}\right)^{2(l_1+2l_2)}\\&\times\max_{i,k\in D}\left[\textbf{E}_Z\left\{\prod_{j=1}^{l_1}\left(g^{(\nu-1)}_{i_j}-\expectation{g^{(\nu-1)}_{i_j}}\right)\right.\right.\\&\times\left.\prod_{b=l_1+1}^{l_1+l_2}\left(g^{(\nu-1)}_{i_b}-\expectation{g^{(\nu-1)}_{i_b}}\right)^2\right\} \\&\left.\times ~(i_j,i_b) ~\text{replaced by}~ (k_i,k_b)\right]+ \mathcal{O}(h_{\nu}^{-1})~.
\end{align*}
However, due to the assumptions, $s_n$ converges to a non-zero limit for $n\to\infty$ and the expectations asymptotically decompose into a product over the $i$s and $j$s which are zero. Thus the expression above and the expression~\eqref{mean_zero_symmetric_polyI} converge to zero  for all $\delta\in(0,1]$. Similar, we find that expression~\eqref{mean_zero_symmetric_polyII} can be  bounded, with at least probability $1-\delta$, by
\begin{align*}
&\frac{C}{\sqrt{\delta}}\left(\frac{||t||_2}{s_n}\right)^{4l}\\&\times\max_{i,j\in  D'}\left[\left(\textbf{E}_Z\left\{\prod_{b=1}^{l}\left(g^{(\nu-1)}_{i_b}-\expectation{g^{(\nu-1)}_{i_b}}\right)^2\right\}\right.\right. \\&\left.-\prod_{b=1}^{l}\expectation{\left(g^{(\nu-1)}_{i_b}-\expectation{g^{(\nu-1)}_{i_b}}\right)^2}\right)\\&\left.\times~ i_b ~\text{replaced by}~ j_b\right] + \mathcal{O}(h_{\nu}^{-1})~.    
\end{align*}
As above, by the assumptions, $s_n$ converges to a non-zero limit for $n\to\infty$ and the expectations asymptotically decompose into a product expectations over the $i$s and $j$s. However, the latter is exactly what is subtracted and thus the expression above and the expression in Eq.\eqref{mean_zero_symmetric_polyII} go to zero for all $\delta\in(0,1]$. Thus in summary, we finally showed that expression~\eqref{absolute_value_psi_with_without_margin} converges to zero for $n\to \infty$ with probability of at least one over the weights. 
\end{proof}
To finally conclude that of $\gamma^{(\nu)}_{j,n}$ in each layer $\nu=2,\dots,k$ become independent random variables, we have to show that this is the case for $\nu=2$. 
\begin{mathcor}
Let $ f^{(2)}_i(x)$, $ g^{(2)}_i(x)$, $s_n^{(2)2}$ and $\gamma^{(2)}_{j,n}$ be as defined in Eq.~\eqref{eq:def_f_mu_j}, \eqref{eq:def_g_mu_j}, \eqref{def_psi_diag_variance} and \eqref{eq:def_psi_gamma}, respectively. Then it follows that $\gamma^{(\nu)}_{j,n}$ are independent random variables.
\end{mathcor}
\begin{proof}
This follows from the definition of the pre-activations and the fact that the preceding layer is deterministic and does not have dropout. 
\end{proof}
Since $\psi^{(\nu)}(t)$ and $\psi^{(\nu)}(t)$ with $\gamma^{(\nu)}_{j,n}$ replaced by a random variable drawn from the \emph{marginalized distribution} of $\gamma^{(\nu)}_{j,n}$, have the same limiting distribution, we can study the limiting distribution of the form by studying the limiting distribution of the latter, which is a simple consequence of Lyapunov's central limit theorem.  
\begin{mathlem}\label{lemma:limit_gaus_margin_psi}
Let the pre-activations $\{ f^{(\nu-1)}_u(x)\}_{u\in\mathbb{N}}$ converge, for $n\to\infty$ in distribution, to independently distributed random variables, let $ f^{(\nu)}_i(x)$, $ g^{(\nu)}_i(x)$, $s_n^{(\nu)2}$ and $\gamma^{(\nu)}_{j,n}$ be as defined in Eq.~\eqref{eq:def_f_mu_j}, \eqref{eq:def_g_mu_j}, \eqref{def_psi_diag_variance} and \eqref{eq:def_psi_gamma}, respectively. Then  $\psi^{(\nu)}(t)$ converges, for $\nu=2,\dots,k$, in distribution to a Gaussian distributed random variable with zero mean and unit variance
\end{mathlem}

\begin{proof}
Due to lemma~\ref{lemma:asymptotic_independence_gamma} it is enough to consider $\psi^{(\nu)}(t)$ with $\gamma^{(\nu)}_{j,n}$ replaced by a random variable drawn from the \emph{marginalized distribution} of $\gamma^{(\nu)}_{j,n}$. The random variables drawn from the marginalized distribution are independent by definition. From lemma~\ref{proof_satisfaction_of_lypunov_condition} it follows that they satisfy the Lyapunov-Condition~\eqref{eq:Lyapounov_condition}, so we can apply Lyapuno's central limit theorem~\ref{lyapunos_clt} to proof that  $\psi^{(\nu)}(t)$ converges in distribution to a Gaussian distributed random variable. Since it has by construction zero mean and unit variance, the claim follows immediately.  
\end{proof}
To finally prove that the pre-activation in layer $\nu$ have Gaussian distributed is a corollary of the results obtained so far.
\begin{mathcor}\label{corollary:activations_are_Gaussian_Processes}
Under the assumptions of lemma~\ref{lemma:limit_gaus_margin_psi}, the pre-activations of layer $\nu$, $\{ f^{(\nu)}_u(x)\}_{u\in\mathbb{N}}$, converge in distribution, for $n\to\infty$,  to independent Gaussian random variables.  
\end{mathcor}
\begin{proof}
From Lemma~\ref{lemma:limit_gaus_margin_psi}, Lemma~\ref{lemma:asymptotic_independence_gamma} and the Cramér-Wold theorem, it follows that every finite marginal of $f^{(\nu)}(x)\in \IR^\infty$  converges in distribution to a Gaussian distributed random variable. From which follows, that $\{ f^{(\nu)}_u(x)\}_{u\in\mathbb{N}}$ converges in distribution to a Gaussian process. Due to lemma~\ref{lemma:Limiting_Covariance_Structure} for all finite index-sets $M\subset \mathbb{N}$ the covariance the pre-activations $f^{(\nu)}_v(x)$ and $f^{(\nu)}_u(x)$ for $u\neq v$ and $u,v\in M$ converges to zero almost surely. Because the pre-activations converges to a uncorrelated Gaussian process, it follows, that  $\{ f^{(\nu)}_u(x)\}_{u\in\mathbb{N}}$ converge to independent random variables.
\end{proof}
\subsection{Limiting Covariance Structure of Layer $\nu$}
In this section we prove that the covariance structure of a finite sub set of the pre-activations in each layer becomes diagonal, in the limit of infinite width, i.e. for $n\to\infty$. 
\begin{mathlem} \label{lemma:Limiting_Covariance_Structure}
Let the pre-activations $\{ f^{(\nu-1)}_u(x)\}_{u\in\mathbb{N}}$ converge, for $n\to\infty$ in distribution to independently distributed random variables,  $ f^{(\nu)}_i(x)$ and $ g^{(\nu)}_i(x)$ be as defined in Eq.~\eqref{eq:def_f_mu_j}, \eqref{eq:def_g_mu_j} . Then, for $\nu=2,\dots,k$ and all finite index-sets $M\subset\mathbb{N}$, 
\begin{align*}
    &\text{Cov}\left(f^{(\nu)}_u,f^{(\nu)}_v\right)= \expectation{f^{(\nu)}_uf^{(\nu)}_v} - \expectation{f^{(\nu)}_u}\expectation{f^{(\nu)}_v}
\end{align*}
converges (almost surely) to zero for $n\to\infty$, whenever $u \neq v$, where $u,v\in M$.
\end{mathlem}
\begin{proof}
To show that the covariance vanishes for $u\neq v$, we use high probability bounds on the weights. Since the expectation with respect to the weights of layer $\nu$ vanishes, we can bound the expression above using Chebyshev's inequality $\mathbf{P}_W(|f(W)| \leq \sqrt{\text{Var}_W(f(W))/\delta})\geq 1-\delta $. The variance is our case is simply the second moment $\expectationW{f^2(W)}$. With probability of at least $1-\delta$, we find that 
\begin{align*}
&\text{Cov}\left(f^{(\nu)}_v,f^{(\nu)}_u\right)  \leq \frac{1}{h_{\nu-1}^2\sqrt{\delta}}\sum_{l\neq k=1}^{h_{\nu-1}} 
 \Big( \expectation{ g^{(\nu-1)}_k g^{(\nu-1)}_l}\\&- \expectation{ g^{(\nu-1)}_k}\expectation{g^{(\nu-1)}_l} \Big)^2+\mathcal{O}(h_{\nu-1}^{-1})~.
\end{align*}
For $\nu=2$, the leading term in the expression above is zero, because by definition, see Eq.~\eqref{eq:def_g_mu_j}, $\textbf{E}\{ g^{(\nu-1)}_k g^{(\nu-1)}_l\} = \textbf{E}\{ g^{(\nu-1)}_k\} \textbf{E}\{g^{(\nu-1)}_l\}$. For $\nu>2$ the sum is of order $\mathcal{O}(h_{\nu-1}^2)$ such that the leading contribution with respect to $h_{\nu-1}$ of expression on the right hand side, is at most of order $\mathcal{O}(1)$. However, since the pre-activations of the proceeding layer converge to independent random variables as $n\to\infty$, i.e. $\textbf{E}\{ g^{(\nu-1)}_k g^{(\nu-1)}_l\} \to \textbf{E}\{ g^{(\nu-1)}_k\} \textbf{E}\{g^{(\nu-1)}_l\}$ for $l\neq k$, the expression above converges to zero for all $\delta\in(0,1]$. 

For $\nu=2,\dots, k$ we rescale $\delta$ by a factor of $2/(|M|(|M|-1))$ and apply the union bound over $u\neq v$ where $u,v\in M$ to find that $\text{Cov}\left(f^{(\nu)}_v,f^{(\nu)}_u\right)$ converges to zero for all $u\neq v$, where $u,v\in M$, with at least probability $1-\delta$. Since $\delta\in(0,1]$ is arbitrary the claim follows. 
\end{proof}

\subsection{Lyapunov Condition for Random Neural Networks}
We prove that a neural network, given by Eq.~\eqref{eq:def_f_mu_j} and \eqref{eq:def_g_mu_j} satisfies the Lyapunov condition,
\begin{align}
    \lim_{n\rightarrow\infty} \frac{1}{s_n^{2+\delta}} \sum_{k=1}^{r_n} \expectation{\left|X_{nk}-\mu_{nk}\right|^{2+\delta}} =0~,\label{eq:Lyapounov_condition}
\end{align}
where $\mu_{nk}=\expectationWithoutVar{X_{nk}}$ and $s_n^{2} = \sum_j\expectationWithoutVar{(X_{nj}-\mu_{nj})^2)}$. The Lyapunov condition is a necessary condition of Lyapunovs  central limit theorem~\cite{Billingsley:prob_and_measure}, which we state for completeness. 
\begin{maththm}\label{lyapunos_clt}
Suppose that for each $n$ the sequence $X_{n1},\dots, X_{nr_n}$
is independent and satisfies $\mu_{nk}=\expectation{X_{nk}}$, $\sigma^2_{nk} = \expectation{\left(X_{nk}-\mu_{nk}\right)^2}$, $s_n^2 = \sum_{k=1}^{r_n} \sigma^2_{nk}$
and let $S_n = X_{n1} + \dots + X_{nr_n}$. If the Lyapunov-condition~\eqref{eq:Lyapounov_condition}
holds for some positive $\delta$, then $(S_n - \sum_{k=1}^{r_n}\mu_{nk}) /s_n \Rightarrow N$.
\end{maththm}
The main argument of the proof relies on the fact, that the activations of the neural network in the numerator and the denominator of condition~\eqref{eq:Lyapounov_condition}, can be bounded independent from the width of the neural network. 

Due to the assumptions on the activation functions of being Lipschitz continuous with Lipschitz constant $L$, we find that
\begin{align}
\left|\phi(u)\right| &\leq \left|\phi(v)\right| + K|u-v|  ~, \label{eq:upper_bound_phi}\\
\left|\phi(u)\right| &\geq \left|\phi(v)\right| - K|u-v| \label{eq:lower_bound_phi}
\end{align}
holds for all $u,v \in \IR$. These bounds will allow us to bound the activation functions. First, we need a simple lemma bounding the activation functions of a neural network with weights drawn form a normal distribution. To formulate the next lemma we make the following assumption. According to the idea of dropout, we set a part of the activation in each layer to zero. We define the set $U^{(\nu)}$ to be the index set of activations in layer $\nu$ not set to zero. The collection of index sets $U^{(\nu)}$, for $\nu=1,\dots,k$, is called an \emph{activation pattern}.
\begin{mathlem}\label{lemma:bound_random_pre_activation}
Let $\alpha>0$, $x\in\IR^N$ be fixed such that $||x||_2^2\leq N \alpha$ and the bias $b$ as well as the weights  $w_i$ of a network pre-activation be drawn from $\mathcal{N}(0,1)$. Then for each finite integer $k$, independent of $N$ and all $U\subseteq U \in \mathcal{P}(\left\{1,...,N\right\})$ there exists a constant $C$ independent of $N$ such that
\begin{align}
    \expectationW{\left|b + \frac{1}{\sqrt{N}}\sum_{j\in U}w_{j}x_j\right|^k} \leq C~. \label{eq:pre_activation_bound}
\end{align}
\end{mathlem}
\begin{proof}
Let $k=2m$, $x_{\perp} = (x_i)_{i\in U}$ and $w_{\perp} = (w_i)_{i\in U}$ such that $\sum_{j\in U}w_{j}x_j = \left<w_{\perp},x_{\perp}\right>$ and $||x_{\perp}||_2^2 \leq ||x||_2^2\leq \alpha N$. Without loss of generality we assume that $1\in U$. We rotate $w_{\perp}$ such that $x_{\perp}$ is given by a vector with $||x_{\perp}||_2$ in the first element and zero elsewhere. This simplifies the expression~\eqref{eq:pre_activation_bound} of the lemma to  
\begin{align*}
 & \sum_{j=1}^{m}{2m\choose 2j}\left(\frac{||x_{\perp}||_2}{\sqrt{N}}\right)^{2(m-j)} c_{2(m-j)} c_{2j}\\
 \leq& \sum_{j=0}^{m}{2m\choose 2j}\alpha^{(m-j)} c_{2(m-j)}c_{2j} \leq C~,
\end{align*}
where $c_{n}$ is the $n$th moment of a Gaussian random variable with zero mean and unit variance. The last inequality in the estimate above is due to the fact that the sum as well as the moments and coefficients are independent of $N$,  proving the claim for even $k$. Let $X$ be a random variable with $\expectationWithoutVar{|X|^{2k}}<\infty$, then it follows that
\begin{align}
\begin{split}
\expectationWithoutVar{|X|^{2m+1}} &= \sqrt{\expectationWithoutVar{|X|^{4m+2}}- \text{Var}\left(|X|^{2m+1}\right)} \\&\leq \sqrt{\expectationWithoutVar{|X|^{4m+2}}}~.
\end{split}\label{eq:bounding_even_moments}
\end{align}
Since all finite moments of a Gaussian random variable exist, the claim follows for odd $k$ from Eq.~\eqref{eq:bounding_even_moments}.
\end{proof}
We now apply lemma~\ref{lemma:bound_random_pre_activation} recursively to the neural network to make the following observation. Since it would only slightly change the absolute values of the constants bounding the expressions in the following and not change their dependence on the layer width, we will keep all activations in the networks, i.e. set $U^{(\nu)} = [h_{\nu}(n)]$, without explicitly mentioning it. 
\begin{mathlem}\label{lemma:finitely_bounded_activations}
Let $\alpha>0$, $x\in\IR^N$ be fixed such that $||x||_2^2\leq N \alpha$, $g_i^{(\nu)}(x)$ be as defined in Eq.~\eqref{eq:def_g_mu_j} and the bias $b_i^{(\nu)}$ as well as the weights $w_{ij}^{(\nu)}$, of a network pre-activation, be drawn from $\mathcal{N}(0,1)$. Then there exists a constant $C<\infty$ such that
$$ \expectationW{\left|g_i^{(\nu)}(x)\right|^k} \leq C \quad \forall \nu = 1,\dots, L~,$$
for each finite, positive integer $k$, where the expectation is over the weights and the bias. 
\end{mathlem}
\begin{proof}
We prove the statement by induction over $\nu$. For $\nu=1$, this follows immediately from lemma~\ref{lemma:bound_random_pre_activation} and Eq.~\eqref{eq:upper_bound_phi} with $v=0$. Thus, the induction hypothesis is that this is true for $\nu$. For $\nu+1$ we find
\begin{align*}
\expectationW{\left|g_i^{(\nu+1)}(x)\right|^k} \leq  \expectationW{\left(K\left|f_i(x)^{(\nu+1)}\right| + \left|\phi(0)\right|\right)^k}  \\=\sum_{j=0}^{k}{k\choose j} K^j \left|\phi(0)\right|^{k-j} \expectationW{\left|f_i(x)^{(\nu+1)}\right|^j}~.
\end{align*}
As $f^{(\nu +1)}_i(x)$ is an affine-linear combination of $g_j^{(\nu)}$ which is bounded according to the induction hypothesis, all absolute moments of $f^{(\nu +1)}_i(x)$ are bounded. This concludes the induction step.
\end{proof}
We facilitate this bound on the absolute moments of the individual moments, to find a high probability bound on polynomials of the absolute value of the activation functions.
\begin{mathlem}\label{lemma:high_prob_bound_activations}
Let $p(x)$ be a polynomial of degree $k$ in $x$. With probability of at least $1-\delta$ it follows, under the assumption of Lemma~\ref{lemma:finitely_bounded_activations}, that there exists a constant $C$, independent of $h_{\nu}(n)$ such that
$$ \left| p\left(\left|g_i^{(\nu)}(x)\right|\right) - \expectationW{p\left(\left|g_i^{(\nu)}(x)\right|\right)}\right| \leq \sqrt{\frac{h_{\nu}(n)}{\delta}}C~,$$
where the expectations are over the weights and bias, holds for all $i=1,\dots,h_{\nu}(n)$ and any finite integer $k$. 
\end{mathlem}
\begin{proof}
We have that $\text{Var}(X)\leq \sqrt{\mathbf{E}{X^{2}}}$ for any random variable $X$ with finite second moment. From this inequality, an expansion of $p$ in $\mathbf{E}|g_i^{(\nu)}(x)|$ and lemma~\ref{lemma:finitely_bounded_activations} we find that there exists a constant $C'$ such that 
$$ \text{Var}_W\left(p\left(\left|g_i^{(\nu)}(x)\right|\right)\right) \leq C'~,$$
where the variance is with respect to the weights. This estimate in combination with Chebyshev's inequality, yields that with probability of at least $1-\delta/h_{\nu}(n)$
$$ \left| p\left(\left|g_i^{(\nu)}(x)\right|\right) - \expectationW{p\left(\left|g_i^{(\nu)}(x)\right|\right)}\right| \leq \sqrt{\frac{h_{\nu}(n)}{\delta}}C~.$$
Applying the union bound over $i$ to the expression above, we find that it holds for all $i=1,\dots,h_{\nu}(n)$ with probability of at least $1-\delta$.
\end{proof}
The following generalization of the lemmas above, will simplify the proof of the final lemma in this section.
\begin{mathcor}\label{corollary:generalizing_to_all_acti_patterns}
Let $g^{(\nu)}_{i,l}(x)$ be the activation function with the activation pattern $U_l^{(\nu')}$ for $\nu'=1,\dots,\nu-1$ and $l\in\mathcal{L}$, where $\mathcal{L}$ is a finite index set. Then lemma~\ref{lemma:finitely_bounded_activations} and lemma~\ref{lemma:high_prob_bound_activations} hold also for polynomials of finite degree in $\{|g^{(\nu)}_{i,l}(x)~|~l\in\mathcal{L}\}$.
\end{mathcor}
\begin{proof}
The claim follows immediately form the fact that lemma~\ref{lemma:bound_random_pre_activation} holds for all activation patterns and the application of it in the proofs of lemma~\ref{lemma:finitely_bounded_activations} and lemma~\ref{lemma:high_prob_bound_activations}.
\end{proof}
With lemma~\ref{lemma:high_prob_bound_activations} and corollary~\ref{corollary:generalizing_to_all_acti_patterns} at hand, we can easily prove that the condition of the Lyapunov CLT is satisfied. In contrast to the lemmas above, the expectation in the following lemma will be over the dropout random variables. 
\begin{mathlem}\label{proof_satisfaction_of_lypunov_condition}
Let $\alpha>0$, $x\in\IR^d$ be fixed such that $||x||_2^2\leq d\alpha$, $\gamma^{(\nu)}_{j,n}(t)$ and $s_n^{(\nu)}(x)$ be as defined in Eq.~\eqref{eq:def_psi_gamma} and~\eqref{def_psi_diag_variance} respectively. We assume that $s_n^{(\nu)}$ does not converge to zero. Then we have that
$$ \lim_{n\to \infty} \left(\frac{1}{s_n^{(\nu)}(x)}\right)^4 \sum^{h_{\nu-1}}_{j=1}\expectation{\left|\gamma^{(\nu)}_{j,n}(t)\right|^4} = 0~,$$
almost surely over the weights.
\end{mathlem}
\begin{proof}
We consider the following sequences
$$ a_n = \sqrt{\frac{\delta}{h_{\nu-1}(n)}}\left(s_n^{(\nu)}\right)^2$$
and
$$ b_n = \frac{\delta}{h_{\nu-1}(n)} \sum^{h_{\nu-1}}_{j=1}\expectation{\left|\gamma^{(\nu)}_{j,n}(t)\right|^4}~.$$
To see that $a_n$ converges we make the following estimate
$$ a_n \leq C \frac{1}{h_{\nu-1}(n)}\sum^{h_{\nu-1}}_{j=1}\sum_{u,v\in M}w_{uj}^{(\nu)}w_{vj}^{(\nu)}t_ut_v~,$$
where we made use of \eqref{def_psi_diag_variance_exp},  lemma~\ref{lemma:high_prob_bound_activations} and corollary~\ref{corollary:generalizing_to_all_acti_patterns} as follows. The expectation value in Eq.~\eqref{def_psi_diag_variance_exp} is over the space of dropout configurations  and can be interpreted as average over activation patterns. Since the corollary~\ref{corollary:generalizing_to_all_acti_patterns} holds for all activation patterns, we arrive at the expression above. The remainder converges  to $C||t||_2^2$ for $n\to\infty$ with probability 1, for all values of $\delta\in[0,1]$.

The convergence of $b_n$ to zero is due to the following observations. By definition we have that  $b_n\geq 0$. Moreover, the summands of the sum over $j$  are polynomials of degree 4. These polynomials consists of terms of the form
\begin{align*}
(-1)^{4-a}\left|\expectation{g_j^{(\nu-1)}(x)}\right|^{4-a}\expectation{\left|g_j^{(\nu-1)}(x)\right|^a}~, \label{eq:estimate_convergence_bn}
\end{align*}
where $a=0,\dots,4~.$ The first two factors in the expression above can be upper bounded by $\mathbf{E}\{|g_j^{(\nu-1)}(x)|^{4-a}\}$. Due to corollary~\ref{corollary:generalizing_to_all_acti_patterns} and lemma~\ref{lemma:high_prob_bound_activations} there exists a constant $C$, such that with probability of at least $1-\delta$ the resulting polynomial can be bounded by $\sqrt{h_{\nu-1}(n)/\delta}~C$ for all $j=1,\dots,h_{\nu-1}(n)$. This leads us to the upper bound
\begin{align*}
    b_n &\leq 
    \sqrt{\frac{\delta}{h_{\nu-1}(n)}}\frac{C}{h_{\nu-1}^2(n)} \\&\times\sum^{h_{\nu}}_{j=1}\sum_{u,v,l,o\in M}w_{uj}^{(\nu)}w_{vj}^{(\nu)}w_{lj}^{(\nu)}w_{oj}^{(\nu)}t_ut_vt_lt_o~.
\end{align*} 
An inspection of an expectation of the expression above, with respect to the weights, in combination with the Markov inequality shows  that it converges to zero for $n\to\infty$ with probability 1 for all values of $\delta\in[0,1]$ and $t\in \IR^{|M|}$.

By assumption of the lemma, $s_n{(\nu)}$ does not converge to zero, neither does $a_n$. Thus since $a_n$ and $b_n$ converge, we can find, applying the union bound, that the limit of $b_n/a_n^2$ as quotient of their limits and thus find that it converges to zero for $n\to \infty$. However, $b_n/a_n^2$ is nothing else then the sequence of the claim. Since we showed that it converges with probability of at least $1-\delta$ over the weights, for all $\delta\in[0,1]$, we proved the claim.
\end{proof}

\newpage
\section{Empirical Observations}\label{app:empirical}

\textbf{Further implementation details} All biases of the networks are set to zero. Training is done on shuffled mini-batches of size $100$ using the standard train-test data split without any further preprocessing. All experiments were run in pytorch \cite{pytorch} using a \verb|Intel(R) Xeon(R) Gold 6126| CPU and a \verb|NVidia GeForce GTX 1080ti| GPU.

\textbf{Correlations in random networks} We analyze the weight correlations and pre-activation correlations of the untrained $\mathcal{H}_{\rm narrow}$ and $\mathcal{H}_{\rm wide}$. The (row-wise) Pearson correlations of the weight matrices (Fig.~\ref{fig:weight_correlations}, top row) are centered around zero with estimation errors that are determined by the width of the respective network. The column-wise weight correlations look the same. To calculate pre-activation correlations (see Eq. \ref{eq:def_f_mu_j}), we run $10{,}000$ forward passes with dropout for a fixed test image. Fig.~\ref{fig:pre_act_correlations} (top row) shows that these correlations are largely similar to the weight correlations---as can be expected. 

\textbf{Correlations in trained networks} For $\mathcal{H}_{\rm narrow}$, we find similar weight correlations for all hidden layers with correlations coefficients that range from $-1$ to $1$ (Fig.~\ref{fig:weight_correlations}, bottom left). While the distributions for $\mathcal{H}_{\rm wide}$ are even more homogeneous across layers, they span only from approximately $-0.5$ to $0.5$ (Fig.~\ref{fig:weight_correlations}, bottom right). More importantly, these distributions resemble the central part of the $\mathcal{H}_{\rm narrow}$ distributions and do by no means collapse to zero, i.e., although network width varies by one order of magnitude, the global dependence structures are similar in both networks. Next, we study the dependencies between pre-activations that are (mainly) induced by the weight dependencies. Fig.~\ref{fig:pre_act_correlations} shows roughly the same dependence pattern as Fig.~\ref{fig:weight_correlations}, however, for $\mathcal{H}_{\rm wide}$ the pre-activation correlation distribution gets broader with layer depth (Fig.~\ref{fig:pre_act_correlations}, bottom right). Intuitively, we can make sense of this observation: the network iteratively withdraws input information to finally only keep what is useful for classification, and less information distributed over a fixed number of neurons means stronger correlations. Moreover, not only pre-activation correlations increase with layer depth but also their variances and thus covariances.

\textbf{Further observations for trained networks} We find that sigmoid activations suppress non-normal distributions, which is in contrast to tanh, ReLU, and linear activation functions. Further experiments with a custom non-linearity that is a proxy to sigmoid suggest that the combination of being constrained and mapping onto only one half-axis might be a critical condition for a Gaussian inducing non-linearity.  

\begin{figure}[hbt]
    \centering
    \includegraphics[width=0.5\textwidth]{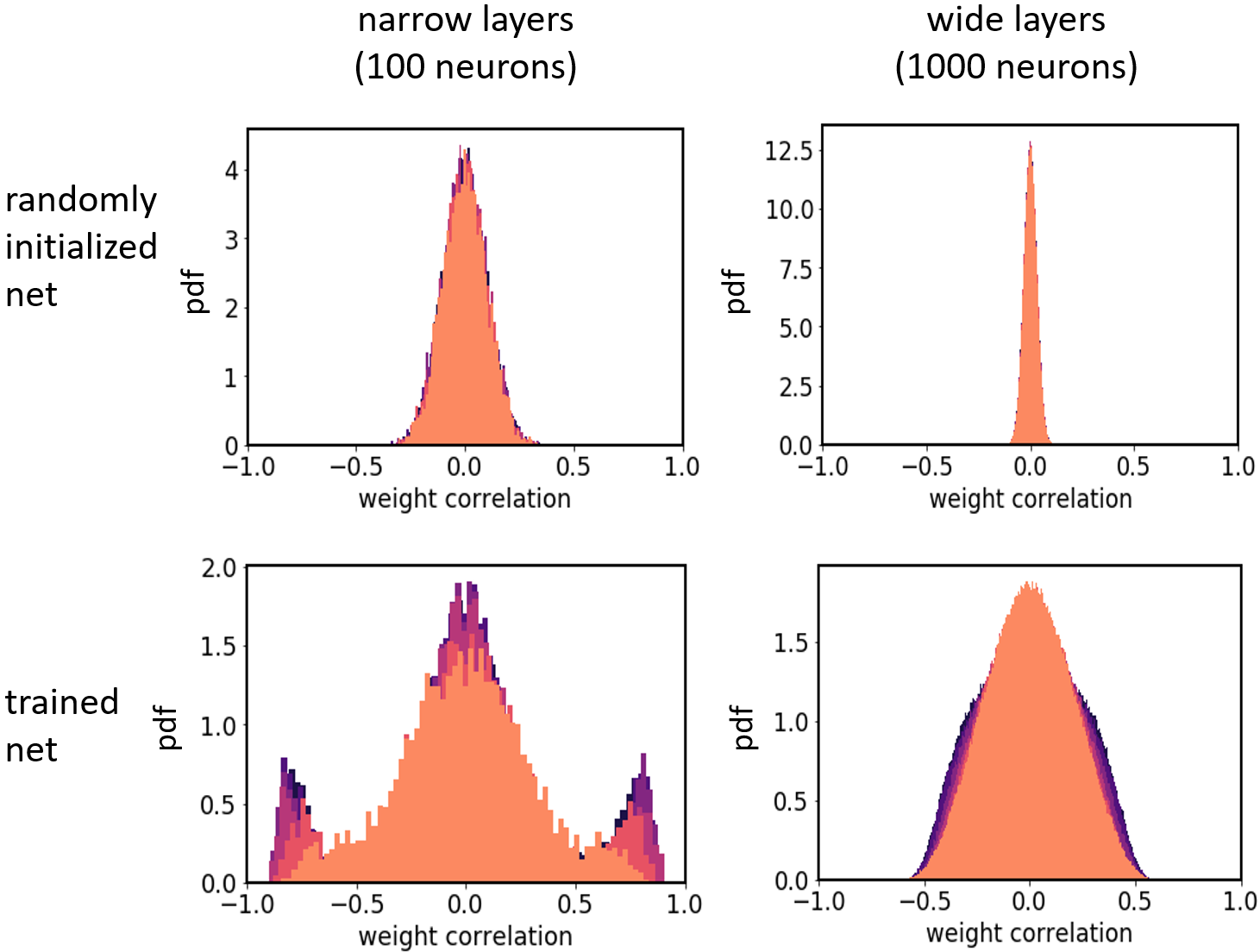}
    \caption{Weight correlations of randomly initialized (top row) and trained (bottom row) networks that are either narrow ($h = 100$, left column) or wide ($h = 1000$, right column). The different hidden layers are color-coded: from first (melon) to last (purple) hidden layer. The Pearson correlation coefficient is used.}
    \label{fig:weight_correlations}
\end{figure}

\begin{figure}[hbt]
    \centering
    \includegraphics[width=0.5\textwidth]{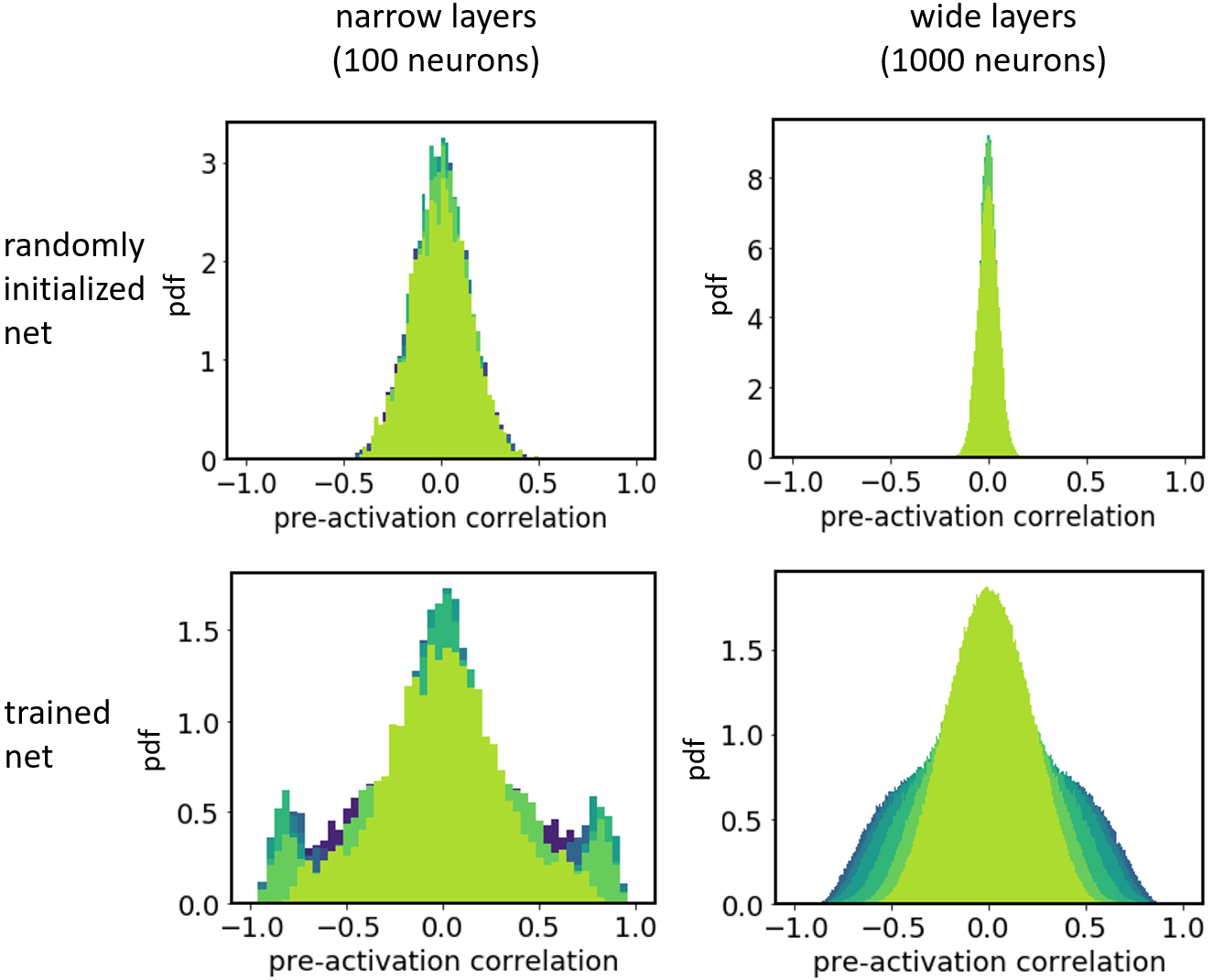}
    \caption{Pre-activation correlations of randomly initialized (top row) and trained (bottom row) networks that are either narrow ($h = 100$, left column) or wide ($h = 1000$, right column). The different hidden layers are color-coded: from first (green) to last (dark blue) hidden layer. The Pearson correlation coefficient is used.}
    \label{fig:pre_act_correlations}
\end{figure}

\clearpage
\newpage
\section{On Correlated Systems}
\label{app:toy_model}

Here, we provide additional information accompanying the calculations in Sec.~\ref{sec:strongly_correlated}.
For instance, the PDF in \eq{eq:zN1PDF} can be calculated using a filter integral:
\begin{align}
    \operatorname{PDF}_{\tilde{f}}(\xi) & \propto
    \int\!\upd x \upd y\, \delta(\xi -x y) \eu^{-x^2/2 -y^2/2}
    \\
    & \propto \int\!\upd x \upd y \upd k\, \eu^{-\iu k (\xi-xy)}\eu^{-x^2/2 -y^2/2}
    \\
    & = \int\!\upd k\, \eu^{-\iu k \xi} \int \!\upd x \upd y\,
    \exp{\left( -\frac{1}{2} \bm{b}^T\!\bm{A}\,\bm{b} \right)}
\end{align}
with the shorthand notations
\begin{equation}
    \bm{b} = \left(
    \begin{array}{c}
         x \\
         y 
    \end{array}\right)
    \quad\text{and}\quad
    \bm{A} = \left(
    \begin{array}{cc}
        1 & -\iu k \\
        -\iu k & 1 
    \end{array}\right)\enspace .
    \label{eq:apx:toyShorthand1}
\end{equation}
Evaluating the inner integrals thus leads to
\begin{align}
    \operatorname{PDF}_{\tilde{f}}(\xi) & \propto
    \int_{-\infty}^{+\infty}\!\upd k\, \frac{\eu^{-\iu k \xi}}{\sqrt{1+k^2}}
    = 2 K_0(|\xi|)\enspace .
\end{align}
The differing prefactor in \eq{eq:zN1PDF} follows from the normalization condition for a PDF. For convenience, we omitted any prefactors here.

A direct extension of this calculation to non-zero mean is challenging.
Instead we provide a heuristic explanation where we include this aspect via
\begin{equation}
    X=\mu_X + \sigma_X \epsilon_X
    \quad\text{with }
    \epsilon_X\sim \mathcal{N}(0,1)
    \quad\text{($Y$ analogous)}\enspace ,
\end{equation}
leading to
\begin{equation}
    X Y= \underbrace{\mu_X \mu_Y}_\text{const.} + \underbrace{\mu_X \sigma_Y\, \epsilon_Y
        + \mu_Y \sigma_X\,\epsilon_X}_\text{``Gaussian''} + \underbrace{\sigma_X\sigma_Y\,\epsilon_X\epsilon_Y}_\text{``tail''}
        \enspace .
    \label{eq:toyWithMean}
\end{equation}
Neglecting the first term as constant, the two middle summands follow a Gaussian behavior while the last one contains a product giving rise to exponential tails.
While this heuristic breakdown ignores the correlations of the twice occurring $\epsilon$'s it qualitatively captures the behavior of $XY$.
Indeed, we find a stronger emphasis towards exponential tails for $\sigma>\mu$ and for Gaussian behavior the other way around.
For illustration,  Fig.~\ref{fig:toyHistogram1} shows the empirical distribution of $Z$ for $\mu_X=\mu_Y=0$ (left) and $\mu_X=\mu_Y=10$ (right).
As a guide we added a numerically obtained PDF as well as an approximation following the logic of our explanation for \eq{eq:toyWithMean}.
In this second case we used the addition of two independent random variables $\tilde Z = \tilde X + \tilde Y$, where $\tilde X \sim N(\mu_X \mu_Y, \sqrt{(\sigma_X \mu_Y)^2 + (\sigma_Y \mu_X)^2})$ is a Gaussian which models the first three terms in \eq{eq:toyWithMean}.
For $\tilde Y$ we use an exponential distribution,
\begin{equation}
    \operatorname{PDF}_{\tilde Y}(\xi)=\frac{1}{2\sigma_X \sigma_Y}\exp{\left(-\frac{|\xi|}{\sigma_X \sigma_Y}\right)}
    \enspace ,
\end{equation}
designed to capture the tail properties of the Bessel function.
Here, we neglected the additional $|\xi|^{-1/2}$ dependence of the asymptotic, Eq.~\eqref{eq:toyGenAsympt}, to keep the resulting integrals of a Gaussian type.
This allows us to give an explicit expression for $\tilde Z$,
\begin{equation}
\begin{split}
    \operatorname{PDF}_{\tilde Z}(\xi) = \frac{\eu^{\frac{\sigma_1^2}{2 \sigma_2^2}}}{4\sigma_2}
    \Bigg( &
    \eu^{\frac{\mu-\xi}{\sigma_2}} \operatorname{Erfc}\left(\frac{\sigma_1^2+(\mu-\xi)\sigma_2}{\sqrt{2} \sigma_1 \sigma_2} \right)
    \\ + & 
    \eu^{\frac{\xi-\mu}{\sigma_2}} \operatorname{Erfc}\left(\frac{\sigma_1^2+(\xi-\mu)\sigma_2}{\sqrt{2} \sigma_1 \sigma_2} \right)
    \Bigg)
    \enspace .
\end{split}
\end{equation}
Therein we used the short hand notations $\mu=\mu_x \mu_Y$, $\sigma_1=\sqrt{(\sigma_X \mu_Y)^2 + (\sigma_Y \mu_X)^2})$ and $\sigma_2= \sigma_X \sigma_Y $. Also,
\begin{equation}
    \operatorname{Erfc}(\zeta) = 1 - \operatorname{Erf}(\zeta) 
    = 1 - \frac{2}{\sqrt{\pi}}\int_0^\zeta \!\upd t\, \eu^{-t^2}
\end{equation}
denotes the complementary error function.
As can be seen in Fig.~\ref{fig:toyHistogram1} this approximation roughly captures the exponential tails.
Furthermore, we find that for $\mu>\sigma$ the tail behavior is suppressed and the distribution becomes asymmetric, a property we similarly observe in the real data, see Fig.~\ref{fig:pre_act_distributions}.
\begin{figure}[bt]
    \includegraphics[width=0.47\columnwidth]{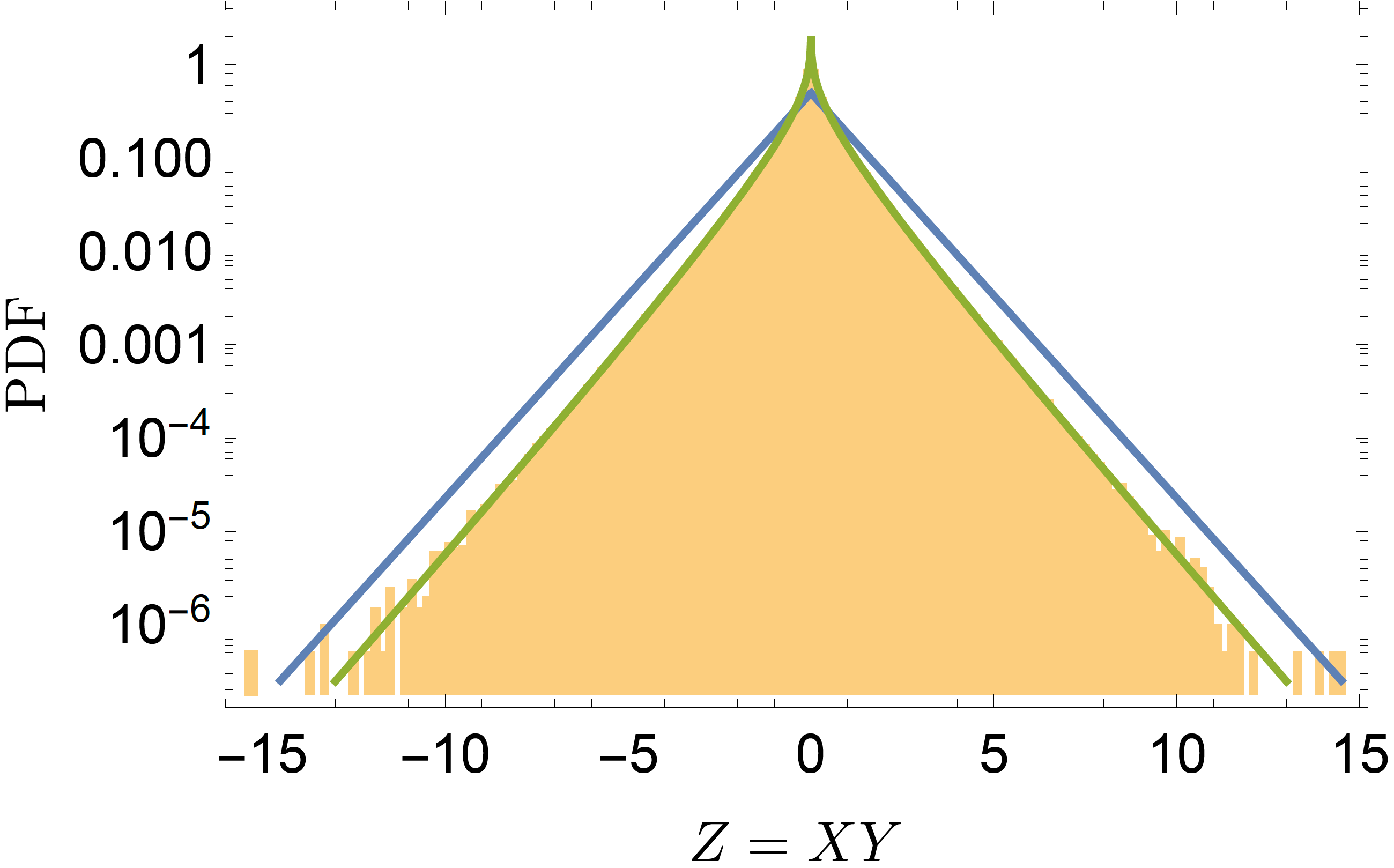}
    \hfill
    \includegraphics[width=0.47\columnwidth]{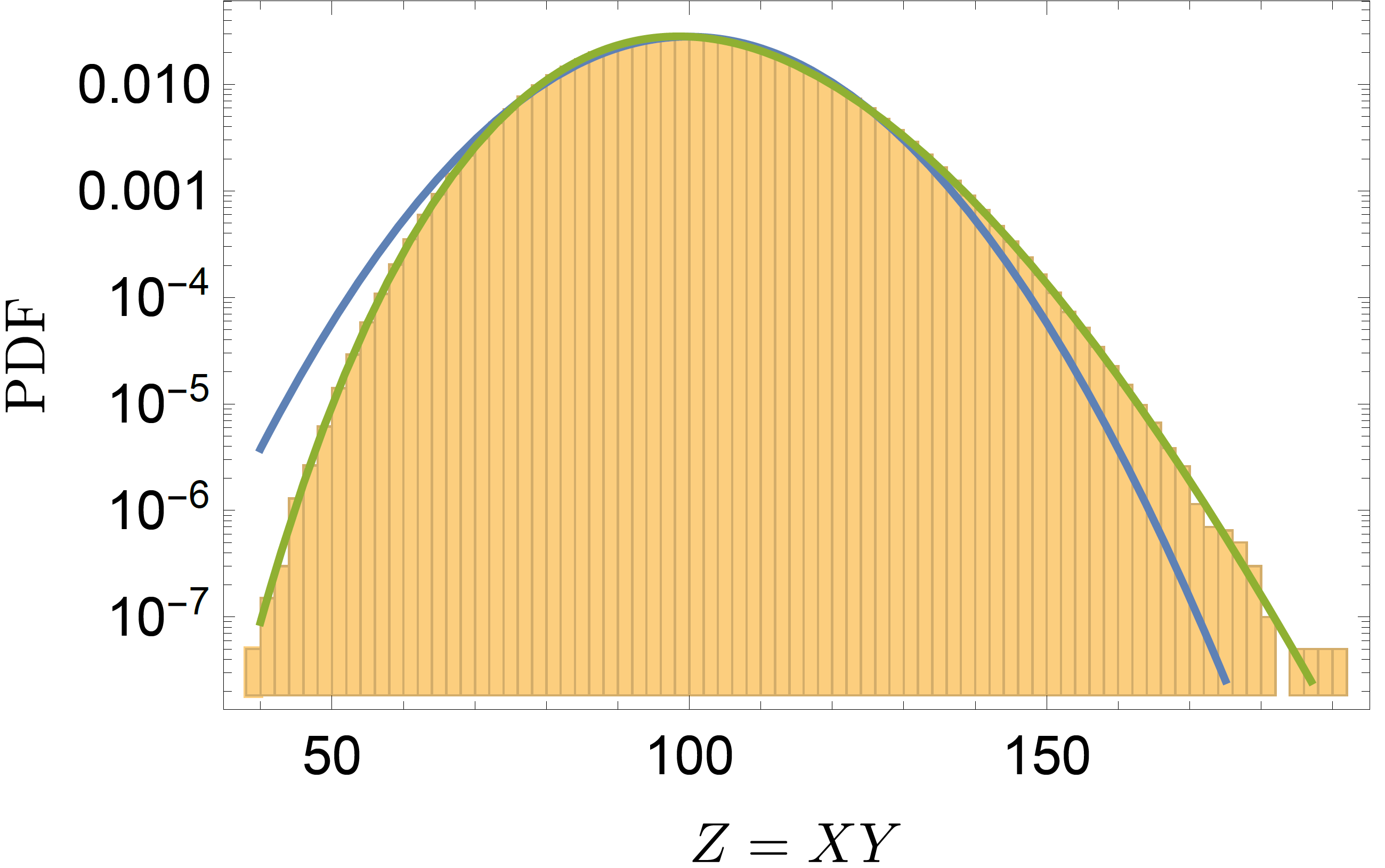}
    \caption{Logarithmic visualization of the PDF for the product of two Gaussian random variables $X,Y$. Shown in blue and green are an approximation to the PDF (see text) and an exact numerical result respectively. Left side shows $\mu_X=\mu_Y=0$ and right $\mu_X=\mu_Y=10$, in both panels $\sigma_X=\sigma_Y=1$.}
    \label{fig:toyHistogram1}
\end{figure}

The extension of the $Z=X Y$ toy model to a sum of random variables $ Z=\sum_{i=1}^h X_i Y_i$ can be treated, at least theoretically, in the same way as the original problem:
\begin{equation}
\begin{split}
    \operatorname{PDF}_{Z}(\xi) & \propto
    \int\!\upd \bm x \upd \bm y\, \delta(\xi -{\bm x}^T \bm y) \eu^{-{\bm x}^T \bm{\Sigma}_X^{-1} \bm x/2 -{\bm y}^T \bm{\Sigma}_Y^{-1} \bm y/2}
    \\
    & = \int\!\upd k\, \eu^{-\iu k \xi} \int \!\upd \bm x \upd \bm y\,
    \exp{\left( -\frac{1}{2} \bm{b}^T \bm{A} \bm{b} \right)}
    \enspace ,
\end{split}
\end{equation}
where in this case $\bm b \in \mathbb{R}^{2h}$ is instead a stacked vector of both $\bm x$ and $\bm y$, compare \eq{eq:apx:toyShorthand1}. Furthermore,
\begin{equation}
    A = \left(
    \begin{array}{cc}
        \Sigma_X^{-1} & -\iu k \mathds{1} \\
        -\iu k \mathds{1} & \Sigma_Y^{-1} 
    \end{array}\right)\,.
\end{equation}
Evaluating the inner integrals gives rise to
\begin{equation}
    \operatorname{PDF}_{Z}(\xi) \propto \int_{-\infty}^{+\infty}\!\upd k \,
    \frac{\eu^{-\iu k \xi}}{\sqrt{\operatorname{det}{\left(\Sigma_{\tilde z}^{-1} \Sigma_{\tilde g}^{-1} + k^2 \mathds{1}
    \right)}}}
    \enspace ,
\end{equation}
where we used the identity
\begin{equation}
    \operatorname{det}
    \left(
     \begin{array}{cc}
        E & B \\
        C & D 
    \end{array}\right)
    = \operatorname{det}\left(ED-BC\right)
\end{equation}
valid for arbitrary square matrices $B,C,D,E$ as long as $CD-DC=0$ holds.
Assuming that the matrix $\Sigma_{\tilde z}^{-1} \Sigma_{\tilde g}^{-1}$ is diagonalizable with (positive) eigenvalues $\sigma_i^{2}$ leads to
\begin{equation}
    \operatorname{PDF}_{Z}(\xi) \propto \int_{-\infty}^{+\infty}\!\upd k \,
    \frac{\eu^{-\iu k \xi}}{\sqrt{\prod_{i=1}^h \left(\sigma_i^{2} + k^2
    \right)}}\enspace .
    \label{eq:toyGeneral}
\end{equation}
Depending on the eigenvalue spectrum the resulting function can exhibit quite different tail behaviors.
To briefly illustrate this, let us make two remarks:
If we assume doubly degenerate eigenvalues we can ``ignore'' the square root and \eq{eq:toyGeneral} instead has poles of \textit{first} order at the positions $k=\pm \iu \sigma_i$.
Based on the residue theorem we then find
\begin{equation}
    \operatorname{PDF}_{Z}(\xi) \propto \sum_{i=1}^h \frac{\eu^{-\sigma_i |\xi|}}{2\sigma_i \prod_{j \neq i}^h (\sigma_j^2+\sigma_i^2)}\,.
    \label{eq:toyLaplace}
\end{equation}
In a loose sense this might be seen as a discrete variant of a Laplace transform and therefore has a similar variety of outcomes.
Those could be largely restricted if one assumes the spectrum of the $\sigma_i$ to be bounded.
In the body of the paper we omitted this discussion and instead chose an easier example closer to the intended application.

The correlation expressed in Eq.~\eqref{eq:corrToy} can  be seen as stemming from a multivariate Gaussian with zero mean and covariance matrix
\begin{equation}
    \bm{\Sigma}_{ij}=\left\{\begin{array}{ll}
        1 & \quad i=j \\
        c & \quad i\neq j 
    \end{array}\right.\enspace .
\end{equation}
While this matrix has a very simple structure with only two distinct eigenvalues of $1+(h-1)c$ and $1-c$, the latter is $h-1$ fold degenerate.
Therefore, our considerations from Eq.~\eqref{eq:toyLaplace} do not directly apply.
Instead, we use this model to build the weight matrices for the random network shown in Fig.~\ref{fig:net_with_correlated_init}.
To this end, we draw for each $\bm{W}_\nu$ two sets of such multivariate Gaussian distributed vectors $\{ \bm{a}_i,\bm{b}_i\in \mathbb{R}^h \}$.
These vectors form two matrices $\bm{A}=(a_1,\ldots, a_h)\in \mathbb{R}^{h\times h}$ ($\bm{B}$ analogous) with correlated rows such that each $\bm{W}$ is given by $\bm{W}_\nu=(\bm{A}_\nu+\bm{B}_\nu^T)/(2 \sqrt{q h})$, where $q=1-p$ denotes the keep rate.
This way we ensure independent correlations both among rows and columns of the matrices.

\end{document}